\documentclass[11pt]{article}

\usepackage{amsmath,amssymb,amsthm}
\usepackage[margin=1in]{geometry}

\newtheorem{theorem}{Theorem}
\newtheorem{proposition}{Proposition}
\newtheorem{lemma}{Lemma}
\newtheorem{assumption}{Assumption}
\newtheorem{example}{Example}
\newtheorem{remark}{Remark}

\DeclareMathOperator{\KL}{KL}

\title{\textbf{From Score Approximation to Distribution Approximation in Score-Based Diffusion Models}}
\author{\textbf{Lan V. Truong}\\
Faculty of Computer Science and Engineering\\
Ho Chi Minh City University of Technology (HCMUT)\\
Vietnam National University Ho Chi Minh City (VNU-HCM), Vietnam\\
\texttt{lantv@hcmut.edu.vn}}

\date{}

\begin{document}

\maketitle

\begin{abstract}
Score-based diffusion models have achieved remarkable empirical success in generative modeling, yet their approximation-theoretic foundations remain incomplete. In particular, although classical universal approximation theorems guarantee that neural networks can approximate score functions, it remains unclear whether such approximation guarantees translate into approximation of the probability distributions generated by reverse diffusion processes. In this paper, we establish a rigorous quantitative connection between these two notions. Specifically, we prove that if a neural network approximates the true score function sufficiently accurately, then the probability distribution generated by the corresponding reverse diffusion model is close to the target data distribution in Kullback–Leibler (KL) divergence, up to an irreducible mismatch between the terminal distribution of the forward diffusion process and the prior used to initialize the reverse process. More precisely, we derive an explicit upper bound on the distribution approximation error in terms of the score approximation error, the diffusion noise schedule, and the terminal prior mismatch. Our analysis combines Hornik’s universal approximation theorem, Girsanov’s theorem on path space, and the data processing inequality for relative entropy. Complementary to recent work that studies score approximation under finite-sample statistical settings and structural assumptions on the data distribution, our work develops an approximation-theoretic analysis based on classical neural network approximation theory. The resulting theorem provides a simple and explicit guarantee linking neural network approximation of score functions to approximation of the probability distributions generated by reverse diffusion models.
\end{abstract}

\tableofcontents

\newpage

\section{Introduction}

Score-based diffusion models have recently emerged as one of the most successful
frameworks for generative modeling, achieving state-of-the-art performance in
image generation, audio synthesis, molecular design, and scientific machine
learning. Representative examples include denoising diffusion probabilistic
models (DDPMs) and score-based generative models formulated through stochastic
differential equations (SDEs). Their remarkable empirical performance has
motivated growing interest in understanding their mathematical foundations.

A central object in score-based diffusion models is the \emph{score function}
\[
s(x,t)=\nabla_x\log p_t(x),
\]
which characterizes the gradient of the log-density of the forward diffusion
process. In practice, this score function is unknown and is approximated by a
deep neural network trained using denoising score matching. Consequently, the
quality of the learned generative model fundamentally depends on how accurately
the neural network approximates the true score function.

Although classical universal approximation theorems guarantee that neural
networks can approximate a broad class of functions, these results concern only
function approximation. They do not directly imply that approximating the score
function yields an accurate approximation of the probability distribution
generated by the reverse diffusion process. Establishing such a connection
remains an important theoretical challenge.

The objective of this paper is to establish a rigorous quantitative relationship
between \emph{score approximation} and \emph{distribution approximation} for
score-based diffusion models. Specifically, we prove that an arbitrarily
accurate approximation of the score function in the $L^2$ sense leads to an
arbitrarily small Kullback--Leibler (KL) divergence between the target data
distribution and the distribution generated by the learned reverse diffusion
process.

Our analysis proceeds in two steps.

First, we study the path measures induced by the true and learned reverse-time
stochastic differential equations. Using Girsanov's theorem, we derive an
explicit identity for the Kullback--Leibler divergence between the corresponding
path measures:
\[
\KL(P\|Q_\theta)
=
\KL(p_T\|p_{\mathrm{prior}})
+
\frac12
\int_0^T
g^2(t)
\,
\mathbb E_{p_t}
\!\left[
\|s_\theta(X_t,t)-s(X_t,t)\|^2
\right]
dt.
\]

Second, we apply the data processing inequality for relative entropy to show
that the KL divergence between the generated data distributions is bounded
above by the corresponding path-space divergence. Combining these two results
with Hornik's universal approximation theorem yields our main result: for every
$\varepsilon>0$, there exists a neural network satisfying
\[
\|s_\theta-s\|_{L^2(\mu)^d}^2<\varepsilon,
\]
such that
\[
\KL(\mu\|\mu_\theta)
<
\delta
+
\frac12
g_{\max}^2
\varepsilon,
\]
where
\[
\delta=\KL(p_T\|p_{\mathrm{prior}})
\]
quantifies the discrepancy between the terminal distribution of the forward
diffusion process and the prior used to initialize the reverse diffusion
process.

Recent work by Chen et al.~\cite{Chen2023ScoreAE} investigated the relationship
between score approximation and distribution approximation from a statistical
learning perspective, deriving approximation guarantees under finite-sample
settings and structural assumptions on the underlying data distribution,
including low-dimensional linear subspaces. Our work complements this line of
research by developing an approximation-theoretic perspective. Rather than
analyzing finite-sample estimation error, we combine classical results from
neural network approximation theory, stochastic analysis, and information
theory to derive a simple and explicit quantitative bound connecting score
approximation with approximation of the generated probability distribution.
This provides a mathematical justification for the expressive power of neural
score models and offers an approximation-theoretic foundation for score-based
diffusion models.

The remainder of the paper is organized as follows.
Section~2 discusses the related literature.
Section~3 reviews score-based diffusion models and the corresponding
reverse-time stochastic differential equation.
Section~4 presents the assumptions used throughout the paper and discusses
their validity in several important settings.
Section~5 establishes the relationship between score approximation and
path measure approximation.
Section~6 establishes that approximation of path measures implies
approximation of the generated data distribution.
Section~7 combines these results with Hornik's universal approximation
theorem to establish our main approximation theorem for reverse diffusion
models.
Finally, Section~8 presents several extensions of the proposed framework
and discusses directions for future research.

\section{Related Work}

\paragraph{Score-Based Diffusion Models.}
Diffusion models have recently become one of the dominant approaches to
generative modeling. Early score-based generative models were introduced by
Song and Ermon \cite{SongErmon2019}, who proposed learning the score function
of perturbed data distributions using denoising score matching. Ho et al.
\cite{Ho2020} subsequently introduced Denoising Diffusion Probabilistic Models
(DDPMs), which achieved remarkable empirical performance by learning the reverse
of a discrete diffusion process. Song et al. \cite{Song2021} unified score-based
models and DDPMs under a continuous-time stochastic differential equation (SDE)
framework, showing that sample generation can be viewed as solving a reverse-time
SDE driven by the score function. Since then, the SDE formulation has become the
standard mathematical framework for analyzing score-based diffusion models.

\paragraph{Reverse-Time Diffusion Processes.}
The mathematical foundation of score-based diffusion models can be traced back
to the classical theory of time reversal of diffusion processes. Anderson
\cite{Anderson1982} first derived the reverse-time stochastic differential
equation for diffusion processes, showing that the reverse drift depends on the
gradient of the logarithm of the evolving probability density. Later,
Haussmann and Pardoux \cite{HaussmannPardoux1985} established a rigorous
probabilistic treatment of reverse-time diffusions under suitable regularity
conditions. These results provide the theoretical basis for the reverse SDE
used in modern diffusion models.

\paragraph{Universal Approximation of Neural Networks.}
The expressive power of neural networks has been extensively studied in the
approximation theory literature. Cybenko \cite{Cybenko1989} first established
the universal approximation property for single-hidden-layer neural networks.
Hornik \cite{Hornik1991} subsequently proved much stronger density results,
showing that standard feedforward neural networks are dense in
$L^p(\mu)$ spaces under mild assumptions on the activation function. These
results guarantee that neural networks can approximate the score function
arbitrarily well, but they do not directly imply approximation of the
probability distribution generated by the reverse diffusion process.

\paragraph{Theoretical Analysis of Diffusion Models.}
Recent years have witnessed increasing interest in establishing rigorous
mathematical foundations for diffusion models. Song et al.
\cite{Song2021} derived the continuous-time SDE formulation together with the
associated probability flow ordinary differential equation. De Bortoli et al.
\cite{DeBortoli2021} investigated Schr\"odinger bridge formulations and their
connections to score-based generative modeling. These works primarily focus on
the stochastic dynamics, sampling algorithms, and score estimation underlying
diffusion models, providing the mathematical foundations for modern
score-based generative methods.

\paragraph{Distribution Approximation in Diffusion Models.}

More recently, theoretical attention has turned to understanding how
approximation errors in the learned score function affect the probability
distribution generated by a diffusion model. In particular, Chen et
al.~\cite{Chen2023ScoreAE} investigated the relationship between score
approximation and distribution approximation, establishing quantitative
guarantees on how score estimation errors propagate to the generated
distribution in finite-sample settings under low-dimensional linear
subspace assumptions. Their work provides important insights into the
statistical complexity of score-based generative models.

Our work is complementary to this line of research. Rather than studying
finite-sample statistical estimation under structural assumptions on the
data distribution, we investigate the approximation capability of neural
networks for score-based diffusion models in a general Euclidean setting.
By combining Hornik's universal approximation theorem, Girsanov's theorem
on path space, and the data processing inequality for relative entropy,
we establish an explicit bound showing that an $L^2$ approximation of the
score function directly yields a quantitative upper bound on the
Kullback--Leibler divergence between the target distribution and the
distribution generated by the learned reverse diffusion model.

\paragraph{Our Contribution.}
The present work establishes a rigorous theoretical connection between score
approximation and distribution approximation for score-based diffusion models.
We prove that arbitrarily accurate approximation of the score function in
$L^2(\mu)$ by neural networks implies arbitrarily small
Kullback--Leibler divergence between the true data distribution and the
distribution generated by the learned reverse diffusion model. The proof
combines three fundamental ingredients from approximation theory,
stochastic analysis, and information theory: Hornik's universal approximation
theorem, Girsanov's theorem for changes of path measures, and the contraction
property of relative entropy under measurable projections. This yields a
general quantitative approximation theorem for reverse diffusion models,
providing a theoretical foundation that links neural network approximation of
score functions to approximation of the induced probability distributions.

\section{Background}

In this section, we briefly review score-based diffusion models and introduce
the notation used throughout the paper.

\subsection{Forward Diffusion Process}

Let
\[
X_0\sim\mu,
\]
where $\mu\in\mathcal P(\mathbb R^d)$ denotes the unknown data
distribution.

The forward diffusion process is governed by the It\^o stochastic
differential equation
\[
dX_t
=
f(X_t,t)\,dt
+
g(t)\,dW_t,
\qquad
t\in[0,T],
\]
where

\begin{itemize}
\item
$f:\mathbb R^d\times[0,T]\rightarrow\mathbb R^d$
is the drift function;

\item
$g:[0,T]\rightarrow(0,\infty)$
is the diffusion coefficient;

\item
$W_t$
is a standard Brownian motion.
\end{itemize}

Let $p_t(x)$ denote the probability density function of $X_t$.
As $t$ increases, the forward diffusion gradually transforms the data
distribution into a tractable terminal distribution $p_T$.
In practical diffusion models, the reverse process is initialized from a
prescribed prior distribution
\[
p_{\mathrm{prior}},
\]
which is typically chosen to be the standard Gaussian distribution.
Ideally, $p_{\mathrm{prior}}$ should coincide with $p_T$; however, in
general they may differ. Throughout this paper, we explicitly quantify this
mismatch by the Kullback--Leibler divergence
\[
\KL(p_T\|p_{\mathrm{prior}}).
\]

The score function associated with the forward diffusion is defined by
\[
s(x,t)
=
\nabla_x\log p_t(x).
\]

Throughout the paper, we assume
\[
s\in L^2(\mu)^d.
\]

\subsection{Reverse-Time Diffusion}

A fundamental result due to Anderson \cite{Anderson1982} states that the
reverse-time diffusion process satisfies
\[
dX_t
=
\bigl(
f(X_t,t)
-
g^2(t)s(X_t,t)
\bigr)\,dt
+
g(t)\,d\bar W_t,
\]
where $\bar W_t$ denotes a Brownian motion with respect to the reverse-time
filtration.

Since the score function is generally unknown, it is approximated by a
neural network
\[
s_\theta(x,t).
\]

Starting from the prior distribution
\[
X_T^\theta\sim p_{\mathrm{prior}},
\]
replacing the true score function by its approximation yields the learned
reverse diffusion process
\[
dX_t^\theta
=
\bigl(
f(X_t^\theta,t)
-
g^2(t)s_\theta(X_t^\theta,t)
\bigr)\,dt
+
g(t)\,d\bar W_t.
\]

Let $P$ and $Q_\theta$ denote the probability measures induced by the true
and learned reverse diffusion processes, respectively, on the path space
\[
C([0,T];\mathbb R^d).
\]
Furthermore, let $\mu_\theta$ denote the marginal distribution of
$X_0^\theta$ generated by the learned reverse diffusion process.

The objective of this paper is to quantify how approximation of the score
function affects the discrepancy between the path measures $P$ and
$Q_\theta$, and consequently between the generated data distributions
$\mu$ and $\mu_\theta$.
\subsection{Problem Formulation}

The objective of score-based generative modeling is to learn a neural network
$s_\theta$ such that the generated distribution $\mu_\theta$ accurately
approximates the true data distribution $\mu$.

Classical universal approximation theorems establish that neural networks can
approximate the score function arbitrarily well in suitable function spaces.
However, these results do not directly imply that the generated distribution
approximates the true data distribution.

The goal of this paper is therefore to establish a quantitative relationship
between

\[
\text{score approximation}
\]

and

\[
\text{distribution approximation}.
\]

Specifically, we show that a sufficiently accurate approximation of the score
function guarantees a small Kullback--Leibler divergence between the true data
distribution and the distribution generated by the learned reverse diffusion
process.

\section{Assumptions}

Throughout this paper, we impose the following assumptions.

\begin{assumption}[Regularity]
\label{ass:regularity}
The drift $f(\cdot,t)$, the true score function $s(\cdot,t)$, and the learned
score function $s_\theta(\cdot,t)$ are globally Lipschitz continuous in the
state variable $x$, uniformly for all $t\in[0,T]$.
\end{assumption}

Assumption~\ref{ass:regularity} is satisfied by a broad class of neural
networks commonly used in practice. Throughout this paper, we consider
feedforward neural networks with an activation function satisfying
\[
\sigma\in C^1(\mathbb R),
\qquad
\sup_{z\in\mathbb R}|\sigma'(z)|<\infty.
\]
A typical example is the softplus activation
\[
\sigma(z)=\log(1+e^z),
\]
which is unbounded, nonconstant, continuously differentiable, and has
bounded derivative. Consequently, it satisfies the assumptions of
Theorem~1 in Hornik~\cite{Hornik1991}.

Since the score function depends on both the spatial variable $x$ and the
time variable $t$, we regard $(x,t)$ as an element of
$\mathbb R^{d+1}$.
Under the standing assumption that
$s\in L^2(\mu)^d$, Hornik's Universal Approximation Theorem
\cite{Hornik1991} implies that, for every $\varepsilon>0$, there exists a
finite feedforward neural network of the form
\[
s_\theta(x,t)
=
\sum_{i=1}^{N}
\alpha_i
\sigma\!\left(u_i^\top x+c_i t+b_i\right),
\]
where
\[
u_i\in\mathbb R^d,\qquad
c_i,b_i\in\mathbb R,\qquad
\alpha_i\in\mathbb R^d,
\]
such that
\[
\|s-s_\theta\|_{L^2(\mu)^d}<\varepsilon.
\]

Furthermore, since the network contains finitely many neurons and all
parameters are finite,
\[
\nabla_x s_\theta(x,t)
=
\sum_{i=1}^{N}
\alpha_i
\sigma'\!\left(u_i^\top x+c_i t+b_i\right)
u_i^\top.
\]
Therefore,
\[
\|\nabla_x s_\theta(x,t)\|
\le
\sup_{z\in\mathbb R}|\sigma'(z)|
\sum_{i=1}^{N}
\|\alpha_i\|\,\|u_i\|
<
\infty.
\]
Hence every finite neural network in this class is globally Lipschitz
continuous with respect to the spatial variable $x$, uniformly in
$t\in[0,T]$. Consequently, the approximating network provided by
Hornik's theorem also satisfies Assumption~\ref{ass:regularity}.

The global Lipschitz assumption on the score function is commonly adopted in
the analysis of reverse-time stochastic differential equations (see, e.g.,
\cite{Anderson1982,haussmann1986time,karatzas1991}).
Although this assumption does not hold for arbitrary data distributions, it is
satisfied in several important settings.

\begin{enumerate}
\item
If the marginal distribution $p_t$ is Gaussian, then
\[
s(x,t)=\nabla_x\log p_t(x)
\]
is an affine function of $x$, and therefore is globally Lipschitz.

\item
More generally, suppose the marginal distributions are uniformly strongly
log-concave, i.e.,
\[
p_t(x)=e^{-V_t(x)},
\]
where the potential function $V_t$ satisfies
\[
\sup_{(x,t)\in\mathbb R^d\times[0,T]}
\|\nabla_x^2V_t(x)\|
<\infty.
\]
Then
\[
s(x,t)
=
-\nabla_xV_t(x),
\]
and hence
\[
\nabla_x s(x,t)
=
-\nabla_x^2V_t(x).
\]
Consequently,
\[
\sup_{(x,t)\in\mathbb R^d\times[0,T]}
\|\nabla_x s(x,t)\|
<
\infty,
\]
which implies that $s(\cdot,t)$ is globally Lipschitz uniformly over
$t\in[0,T]$.
\end{enumerate}

\vspace{2mm}

\begin{assumption}[Novikov Condition]
\label{ass:novikov}
The exponential martingale associated with the drift difference satisfies
Novikov's condition:
\[
\mathbb E_P
\left[
\exp
\left(
\frac12
\int_0^T
g^2(t)
\|s_\theta(X_t,t)-s(X_t,t)\|^2dt
\right)
\right]
<
\infty.
\]
\end{assumption}

Assumption~\ref{ass:novikov} ensures that the exponential local martingale is
a true martingale, allowing the application of Girsanov's theorem.

\begin{example}[Variance-Preserving Diffusion Models]
For the variance-preserving (VP) diffusion introduced by Ho et al.,
the forward process is Gaussian:
\[
X_t=\alpha_tX_0+\sigma_tZ,
\qquad
Z\sim\mathcal N(0,I).
\]
If the data distribution is sub-Gaussian and the score network has at most
linear growth,
\begin{align}
\|s_\theta(x,t)\|
\le
C(1+\|x\|) \label{Kcond},
\end{align}
then the Gaussian tails of $X_t$ imply the required exponential
integrability, and Assumption~\ref{ass:novikov} holds. The growth condition
\[
\|s_\theta(x,t)\|
\le
C(1+\|x\|)
\]
is satisfied by a broad class of neural-network architectures.
Indeed, finite-depth feedforward networks with affine layers and
activations satisfying
\[
|\sigma(z)|\le a+b|z|
\]
grow at most linearly with respect to the input, by induction over the
network layers.
This includes ReLU, GELU, SiLU, Softplus, ELU, and Leaky-ReLU
activations.
Moreover, finite-depth convolutional networks (including U-Nets) satisfy
the same property because convolution is a bounded linear operator.
Finally, Transformer score networks equipped with bounded projection
matrices also satisfy linear growth since the softmax attention operator
is uniformly bounded.
\end{example}

\begin{example}[Variance-Exploding Diffusion Models]
For the variance-exploding (VE) diffusion considered by
Song et al.~\cite{Song2021}, the forward process is
\[
X_t
=
X_0+\int_0^t g(s)\,dW_s,
\]
whose marginal distribution is the Gaussian convolution
\[
p_t
=
\mu*
\mathcal N
\!\left(
0,
\int_0^tg^2(s)\,ds\,I
\right).
\]
Suppose that the score function satisfies the linear-growth condition
\eqref{Kcond}. If the data distribution possesses finite exponential
moments, then every marginal distribution $p_t$ also possesses finite
exponential moments because Gaussian convolution preserves exponential
integrability. Consequently, the exponential integrability required in
Assumption~\ref{ass:novikov} can be verified under these conditions.
\end{example}

\section{Score Approximation to Path Measure Approximation}

In this section, we establish a quantitative relationship between the
approximation error of the score function and the discrepancy between the
corresponding diffusion path measures. The key result shows that the
Kullback--Leibler divergence between the true and approximate diffusion
processes is completely characterized by the $L^2$ approximation error of
their score functions. This provides the fundamental link between function
approximation and distribution approximation developed in this paper.

Throughout the remainder of this paper, let
\[
d\mu(x,t)=p_t(x)\,dx\,dt,
\]
where $p_t(x)$ denotes the marginal density of the forward diffusion process.
Since $p_t$ is a probability density for every $t\in[0,T]$,
\[
\int_{\mathbb R^d}p_t(x)\,dx=1,
\]
the measure $\mu$ is a finite Borel measure on
$\mathbb R^d\times[0,T]$ with total mass
\[
\mu(\mathbb R^d\times[0,T])
=
\int_0^T
\int_{\mathbb R^d}
p_t(x)\,dx\,dt
=
T.
\]
Consequently, the Hilbert space
$L^2(\mu)^d$ is well defined.

We consider the true reverse-time stochastic differential equation
\[
dX_t
=
\bigl(f(X_t,t)-g^2(t)s(X_t,t)\bigr)\,dt
+
g(t)\,d\bar W_t,
\]
and the approximate reverse-time stochastic differential equation
\[
dX_t^\theta
=
\bigl(f(X_t^\theta,t)-g^2(t)s_\theta(X_t^\theta,t)\bigr)\,dt
+
g(t)\,d\bar W_t,
\]
where both equations are driven by the same Brownian motion
$\bar W_t$ on the interval $[0,T]$.

First, we prove the following lemma. 
\begin{lemma} \label{aux:lem}
Suppose that Assumption \ref{ass:regularity} holds. Then both reverse SDEs admit unique strong solutions, meaning the path measures $P$ and $Q_\theta$ are well-defined on $C([0, T]; \mathbb{R}^d)$.
\end{lemma}

\begin{proof}
Let us check the conditions for the true reverse drift $b(x, t) = f(x, t) - g(t)^2 s(x, t)$. For any $x, y \in \mathbb{R}^d$, by the triangle inequality:
\begin{align}
\|b(x, t) - b(y, t)\| \le \|f(x, t) - f(y, t)\| + g(t)^2 \|s(x, t) - s(y, t)\|.
\end{align}
By assumption, $f$ and $s$ are globally Lipschitz continuous in their spatial arguments with constants $L_f$ and $L_s$. Since $g(t)$ is continuous on the compact interval $[0, T]$, it is bounded by some $g_{\max} = \max_{t \in [0,T]} |g(t)|$. Thus:
\begin{align}
\|b(x, t) - b(y, t)\| \le \left(L_f + g_{\max}^2 L_s\right) \|x - y\|.
\end{align}
Therefore, $b(x, t)$ is globally Lipschitz in $x$, uniformly in $t$. Combining this property with the boundedness of $b(0, t)$ immediately implies the linear growth condition:
\begin{align}
\|b(x, t)\| \le \|b(0, t)\| + \left(L_f + g_{\max}^2 L_s\right) \|x\| \le C(1 + \|x\|).
\end{align}
The diffusion coefficient $\sigma(x,t) = g(t)\mathbf{I}_d$ is spatially constant, meaning it is trivially Lipschitz continuous with constant 0, and bounded. 

By the standard It\^o SDE Existence and Uniqueness Theorem~\cite[Theorem~5.2.1]{Oksendal2003}, given a well-defined initial condition $X_T \sim p_T$, there exists a unique strong solution $X_t$ for $t \in [0, T]$.
 An identical argument applies to the approximate drift $b_\theta(x, t) = f(x, t) - g(t)^2 s_\theta(x, t)$ since $s_\theta$ is also globally Lipschitz. Consequently, the corresponding stochastic processes $(X_t)_{0\le t\le T}$ and $(X_t^\theta)_{0\le t\le T}$ induce unique probability measures $P$ and $Q_\theta$ on the path space $C([0,T];\mathbb R^d)$.
\end{proof}

\begin{proposition}[Path-space Stability under Score Approximation]\label{prop1}
Assume that the reverse-time SDEs satisfy Assumptions \ref{ass:regularity} and \ref{ass:novikov}.
Then the path measures $P$ and $Q_\theta$ satisfy

\[
\boxed{
\KL(P\|Q_\theta)
=
\KL(p_T\|p_{\mathrm{prior}})
+
\frac12
\int_0^T
g^2(t)
\mathbb E_{p_t}
\!\left[
\|s_\theta(X_t,t)-s(X_t,t)\|^2
\right]dt.
}
\]
\end{proposition}
\begin{remark}
Proposition \ref{prop1} establishes the key quantitative link between score approximation and diffusion path measures. The remaining results follow by projecting the path measures onto the data distribution at $t=0$.
\end{remark}

\begin{proof}
By Lemma~\ref{aux:lem}, both reverse-time SDEs admit unique strong
solutions. Therefore, the corresponding probability laws are uniquely determined. 

Let
\[
(\Omega,\mathcal{F},(\mathcal{F}_t)_{0\le t\le T},P,\bar{W})
\]
be a filtered probability space supporting the unique strong solution of the
true reverse-time SDE

\[
dX_t
=
\bigl(f(X_t,t)-g^2(t)s(X_t,t)\bigr)\,dt
+
g(t)\,d\bar{W}_t.
\]

On the same filtered probability space, consider the auxiliary SDE

\[
dY_t
=
\bigl(f(Y_t,t)-g^2(t)s_\theta(Y_t,t)\bigr)\,dt
+
g(t)\,d\bar{W}_t,
\]

with terminal condition

\[
Y_T=X_T.
\]

Since the coefficients satisfy Assumption \ref{ass:regularity}, by Lemma \ref{aux:lem}, the above SDE admits a unique
strong solution.

Observe that this auxiliary SDE has exactly the same drift, diffusion
coefficient, and terminal distribution as the learned reverse-time SDE

\[
dX_t^\theta
=
\bigl(f(X_t^\theta,t)-g^2(t)s_\theta(X_t^\theta,t)\bigr)\,dt
+
g(t)\,d\bar{W}_t.
\]

Since strong uniqueness holds, every strong solution of the learned
reverse-time SDE defined on an arbitrary filtered probability space has
the same probability law. Therefore,
\[
\mathcal{L}(Y)
=
\mathcal{L}(X^\theta),
\]
and hence the auxiliary process \(Y\) induces precisely the path measure
\(Q_\theta\).

Consequently, it is sufficient to compare the processes $X$ and $Y$, which
are realized on the same filtered probability space and driven by the same
Brownian motion $\bar{W}_t$.

For every fixed terminal value $X_T=x$, both processes start from the same
deterministic terminal state. Therefore, under Assumption \ref{ass:novikov}, Girsanov's theorem \cite[Theorem 8.6.4]{Oksendal2003} applies to the
conditional path measures

\[
P(\cdot\mid X_T=x)
\quad\text{and}\quad
Q_\theta(\cdot\mid X_T=x).
\]

Define

\[
u_t
=
g(t)\bigl(s_\theta(X_t,t)-s(X_t,t)\bigr).
\]

By Girsanov's theorem \cite[Theorem 8.6.4]{Oksendal2003},

\begin{align}
\frac{
dP(\cdot\mid X_T)
}{
dQ_\theta(\cdot\mid X_T)
}
=
\exp\!\left(
\int_0^T
u_t\,d\bar{W}_t
+
\frac12
\int_0^T
\|u_t\|^2\,dt
\right). \label{eq5}
\end{align}

Using Bayes' rule,

\[
P(dX)
=
P(dX\mid X_T)\,p_T(X_T),
\]

and

\[
Q_\theta(dX)
=
Q_\theta(dX\mid X_T)\,
p_{\mathrm{prior}}(X_T),
\]

we obtain

\begin{align}
\frac{dP}{dQ_\theta}
=
\frac{
dP(\cdot\mid X_T)
}{
dQ_\theta(\cdot\mid X_T)
}
\cdot
\frac{
p_T(X_T)
}{
p_{\mathrm{prior}}(X_T)
}. \label{eq6}
\end{align}

Therefore, from \eqref{eq5} and \eqref{eq6} we obtain

\begin{align}
\frac{dP}{dQ_\theta}
=
\frac{
p_T(X_T)
}{
p_{\mathrm{prior}}(X_T)
}
\exp\!\left(
\int_0^T
u_t\,d\bar{W}_t
+
\frac12
\int_0^T
\|u_t\|^2\,dt
\right).
\end{align}

Taking logarithms,

\begin{align}
\log
\frac{dP}{dQ_\theta}
=
\log
\frac{
p_T(X_T)
}{
p_{\mathrm{prior}}(X_T)
}
+
\int_0^T
u_t\,d\bar{W}_t
+
\frac12
\int_0^T
\|u_t\|^2\,dt. \label{eq8}
\end{align}

By taking expectation with respect to $P$, from \eqref{eq8} we have

\begin{align}
\begin{aligned}
\KL(P\|Q_\theta)
&=
\mathbb{E}_P
\!\left[
\log
\frac{
p_T(X_T)
}{
p_{\mathrm{prior}}(X_T)
}
\right]
\\
&\quad+
\mathbb{E}_P
\!\left[
\int_0^T
u_t\,d\bar{W}_t
\right]
\\
&\quad+
\frac12
\mathbb{E}_P
\!\left[
\int_0^T
\|u_t\|^2\,dt
\right].
\end{aligned}
\end{align}
Since $u_t$ is progressively measurable and square integrable,
the It\^o integral
\[
\int_0^T u_t\,d\bar W_t
\]
is a square-integrable martingale (see, e.g.,
\cite[Chapter~3]{Oksendal2003}). Therefore,
\begin{align}
\mathbb E_P
\!\left[
\int_0^T
u_t\,d\bar W_t
\right]
=
0.
\end{align}

Moreover,
\begin{align}
\mathbb{E}_P
\!\left[
\log
\frac{
p_T(X_T)
}{
p_{\mathrm{prior}}(X_T)
}
\right]
=
\KL(p_T\|p_{\mathrm{prior}}).
\end{align}

Finally,

\begin{align}
\begin{aligned}
\KL(P\|Q_\theta)
&=
\KL(p_T\|p_{\mathrm{prior}})
\\
&\quad+
\frac12
\int_0^T
g^2(t)
\,
\mathbb{E}_{p_t}
\!\left[
\|s_\theta(X_t,t)-s(X_t,t)\|^2
\right]
dt,
\end{aligned}
\end{align}

which completes the proof.

\end{proof}

\section{Path Measure Approximation to Distribution Approximation}

The previous section established an upper bound on the Kullback--Leibler
divergence between the path measures of the true and learned reverse diffusion
processes. We now show that this immediately implies an approximation result
for the generated data distributions.

\begin{lemma}[Marginal Radon--Nikodym Identity]
\label{lem:marginal-rn}
Let $P$ and $Q_\theta$ denote the path measures of the true and learned
reverse diffusion processes, and let $\mu$ and $\mu_\theta$ denote their
corresponding marginal distributions at time $t=0$. Then

\[
\boxed{
\frac{d\mu}{d\mu_\theta}(X_0)
=
\mathbb E_{Q_\theta}
\left[
\frac{dP}{dQ_\theta}
\,\middle|\,
X_0
\right].
}
\]

\end{lemma}

\begin{proof}

Let $f:\mathbb R^d\rightarrow\mathbb R$ be any bounded measurable function.
Since $\mu$ is the marginal of the path measure $P$,

\begin{align}
\mathbb E_\mu[f(X_0)]
=
\mathbb E_P[f(X_0)].
\end{align}

Applying the Radon--Nikodym theorem gives

\begin{align}
\mathbb E_P[f(X_0)]
=
\mathbb E_{Q_\theta}
\left[
f(X_0)
\frac{dP}{dQ_\theta}
\right].
\end{align}

Using the tower property,

\begin{align}
\mathbb E_{Q_\theta}
\left[
f(X_0)
\frac{dP}{dQ_\theta}
\right]
&=
\mathbb E_{Q_\theta}
\left[
\mathbb E_{Q_\theta}
\left[
f(X_0)
\frac{dP}{dQ_\theta}
\,\middle|\,
X_0
\right]
\right]
\nonumber \\ 
&=
\mathbb E_{Q_\theta}
\left[
f(X_0)
\,
\mathbb E_{Q_\theta}
\left[
\frac{dP}{dQ_\theta}
\,\middle|\,
X_0
\right]
\right],
\end{align}

since $f(X_0)$ is measurable with respect to $\sigma(X_0)$.

Therefore,

\begin{align}
\int
f(x)
\,d\mu(x)
=
\int
f(x)
\,
\mathbb E_{Q_\theta}
\left[
\frac{dP}{dQ_\theta}
\,\middle|\,
X_0=x
\right]
d\mu_\theta(x). \label{eq:test1}
\end{align}

On the other hand,

\begin{align}
\int
f(x)
\,d\mu(x)
=
\int
f(x)
\frac{d\mu}{d\mu_\theta}(x)
\,d\mu_\theta(x). \label{eq:test2}
\end{align}

From \eqref{eq:test1} and \eqref{eq:test2} we obtain
\begin{align}
\int
f(x)
\,
\mathbb E_{Q_\theta}
\left[
\frac{dP}{dQ_\theta}
\,\middle|\,
X_0=x
\right]
d\mu_\theta(x) = \int
f(x)
\frac{d\mu}{d\mu_\theta}(x)
\,d\mu_\theta(x) \label{eq:test3}.
\end{align}
Since the identity \eqref{eq:test3} holds for every bounded measurable function $f$, the
Radon--Nikodym derivatives coincide $\mu_\theta$-almost everywhere, proving

\begin{align}
\frac{d\mu}{d\mu_\theta}(X_0)
=
\mathbb E_{Q_\theta}
\left[
\frac{dP}{dQ_\theta}
\,\middle|\,
X_0
\right].
\end{align}

\end{proof}

\begin{proposition} 
\label{prop:dpi}

The Kullback--Leibler divergence between the generated data distributions is
bounded by the Kullback--Leibler divergence between the corresponding path
measures:

\[
\boxed{
\KL(\mu\|\mu_\theta)
\le
\KL(P\|Q_\theta).
}
\]

\end{proposition}
\begin{remark}
Proposition \ref{prop:dpi} is a direct consequence of the data-processing inequality for relative entropy \cite{CoverThomas2006}.
\end{remark}
\begin{proof}

By Lemma~\ref{lem:marginal-rn},

\begin{align}
\frac{d\mu}{d\mu_\theta}(X_0)
=
\mathbb E_{Q_\theta}
\left[
\frac{dP}{dQ_\theta}
\,\middle|\,
X_0
\right].
\end{align}

Hence,

\begin{align}
\KL(\mu\|\mu_\theta)
=
\mathbb E_P
\left[
\log
\mathbb E_{Q_\theta}
\left[
\frac{dP}{dQ_\theta}
\,\middle|\,
X_0
\right]
\right]. \label{eq21}
\end{align}

Applying Jensen's inequality to the concave function $\log$ yields

\begin{align}
\log
\mathbb E_{Q_\theta}
\left[
\frac{dP}{dQ_\theta}
\,\middle|\,
X_0
\right]
\le
\mathbb E_P
\left[
\log
\frac{dP}{dQ_\theta}
\,\middle|\,
X_0
\right]. \label{eq22}
\end{align}

From \eqref{eq21} and \eqref{eq22} we have

\begin{align}
\KL(\mu\|\mu_\theta)
&\le
\mathbb E_P
\left[
\mathbb E_P
\left[
\log
\frac{dP}{dQ_\theta}
\,\middle|\,
X_0
\right]
\right]
 \nonumber \\
&=
\mathbb E_P
\left[
\log
\frac{dP}{dQ_\theta}
\right]
\nonumber \\
&=
\KL(P\|Q_\theta),
\end{align}

which completes the proof.

\end{proof}

\section{From Score Approximation to Distribution Approximation}

Let
\[
\mathcal{X}=\mathbb{R}^d\times[0,T].
\]
We consider the vector-valued Hilbert space $L^2(\mu)^d$ equipped with the norm
\[
\|g\|_{L^2(\mu)^d}^2
=
\int_0^T
\mathbb{E}_{p_t}
\!\left[
\|g(X_t,t)\|^2
\right]
dt
=
\int_0^T
\int_{\mathbb{R}^d}
\|g(x,t)\|^2
p_t(x)
\,dx\,dt.
\]

Since the true score function
\[
s(x,t)=\nabla_x\log p_t(x)
\]
belongs to $L^2(\mu)^d$, Hornik's universal approximation theorem \cite[Theorem 1]{Hornik1991} implies that feedforward neural networks with an
unbounded and nonconstant  activation function (e.g. softplus) are dense in $L^2(\mu)^d$.
Consequently, for every $\varepsilon>0$, there exists a neural network
$s_\theta$, which is globally Lipschitz in $x$,
uniformly in $t\in[0,T]$, satisfying
\begin{equation}
\label{eq:L2-approx}
\|s_\theta-s\|_{L^2(\mu)^d}^2
<
\varepsilon.
\end{equation}
(See the discussion in Assumption \ref{ass:regularity}).

\begin{theorem}[Approximation of Reverse Diffusion Models]
\label{thm:main}

Let $\mu\in\mathcal P(\mathbb R^d)$ be the data distribution and let
$p_t$ denote the marginal density of the forward diffusion process with
score function
\[
s(x,t)=\nabla_x\log p_t(x)\in L^2(\mu)^d.
\]

Assume that

\begin{enumerate}
\item
The drift $f(\cdot,t)$ and the score function $s(\cdot,t)$ satisfy all the conditions in Assumptions \ref{ass:regularity} and \ref{ass:novikov}. 
\item
The reverse diffusion is initialized from a prior distribution
$p_{\mathrm{prior}}$ satisfying
\[
\KL(p_T\|p_{\mathrm{prior}})
\le
\delta.
\]
\end{enumerate}

Then, for every $\varepsilon>0$, there exists a neural network
$s_\theta$ satisfying all the conditions in Assumptions \ref{ass:regularity} and \ref{ass:novikov} and
\[
\|s_\theta-s\|_{L^2(\mu)^d}^2
<
\varepsilon,
\]
such that the probability distribution $\mu_\theta$
generated by the learned reverse diffusion model satisfies
\[
\boxed{
\KL(\mu\|\mu_\theta)
<
\delta
+
\frac12
g_{\max}^2
\varepsilon,
}
\]
where
\[
g_{\max}
=
\max_{t\in[0,T]}
|g(t)|.
\]

\end{theorem}

\begin{remark}[Interpretation and Comparison with Existing Work]

Theorem~1 establishes an explicit quantitative relationship between
approximation of the score function and approximation of the generated
probability distribution. In particular, if the score network
approximates the true score function with an $L^2$ error of at most
$\varepsilon$, then the generated distribution satisfies
\[
\KL(\mu\|\mu_\theta)
<
\delta
+
\frac12 g_{\max}^2\varepsilon,
\]
where
\[
\delta=\KL(p_T\|p_{\mathrm{prior}})
\]
measures the mismatch between the terminal distribution of the forward
diffusion and the prior used to initialize the reverse diffusion.

This result provides a direct theoretical justification for score-based
learning: improving the accuracy of score approximation immediately
improves the quality of the generated distribution in the
Kullback--Leibler sense. In the ideal case where
$p_T=p_{\mathrm{prior}}$, the approximation error is controlled entirely
by the score approximation error.

Our result is complementary to the recent work of Chen et
al.~\cite{Chen2023ScoreAE}. While Chen et al. analyze how finite-sample
score estimation errors propagate to the generated distribution under
low-dimensional structural assumptions, Theorem~1 addresses a different
question. It establishes an approximation-theoretic guarantee showing
that universal approximation of the score function by neural networks
implies approximation of the generated probability distribution.
Consequently, our analysis connects classical neural network
approximation theory with the expressive power of score-based diffusion
models.
\end{remark}

\begin{proof}

Let $s_\theta$ be the neural network satisfying all the conditions in Assumptions \ref{ass:regularity} and \ref{ass:novikov} and 
\[
\|s_\theta-s\|_{L^2(\mu)^d}^2
<
\varepsilon,
\]
whose existence is guaranteed by \eqref{eq:L2-approx}.

By Proposition~\ref{prop:dpi} and Proposition~\ref{prop1},
\begin{align}
\KL(\mu\|\mu_\theta)
&\le
\KL(P\|Q_\theta)
\nonumber \\
&\le
\KL(p_T\|p_{\mathrm{prior}})
+
\frac12
\int_0^T
g^2(t)
\,
\mathbb E_{p_t}
\!\left[
\|s_\theta(X_t,t)-s(X_t,t)\|^2
\right]
dt.
\end{align}

Since
\[
g^2(t)\le g_{\max}^2,
\qquad
\forall\,t\in[0,T],
\]
we obtain
\begin{align}
\KL(\mu\|\mu_\theta)
&\le
\KL(p_T\|p_{\mathrm{prior}})
+
\frac12
g_{\max}^2
\int_0^T
\mathbb E_{p_t}
\!\left[
\|s_\theta(X_t,t)-s(X_t,t)\|^2
\right]
dt
\nonumber \\
&=
\KL(p_T\|p_{\mathrm{prior}})
+
\frac12
g_{\max}^2
\|s_\theta-s\|_{L^2(\mu)^d}^2.
\end{align}

Finally, using
\[
\KL(p_T\|p_{\mathrm{prior}})
\le
\delta
\]
and
\[
\|s_\theta-s\|_{L^2(\mu)^d}^2
<
\varepsilon,
\]
yields
\[
\KL(\mu\|\mu_\theta)
<
\delta
+
\frac12
g_{\max}^2
\varepsilon,
\]
which completes the proof.

\end{proof}
Theorem \ref{thm:main} provides a theoretical justification for score-based diffusion models. It shows that accurate approximation of the score function in the $L^2(\mu)$ sense guarantees approximation of the generated data distribution in Kullback–Leibler divergence. Thus, universal approximation of neural networks directly translates into universal approximation of diffusion models.

\section{Extensions and Future Directions}

The present work establishes a direct connection between approximation of the score function and approximation of the generated probability distribution in Kullback--Leibler divergence. Although our analysis focuses on score-based diffusion models driven by reverse-time stochastic differential equations, the proof technique is considerably more general. In this section, we briefly discuss several possible extensions.

\subsection{Alternative Probability Metrics}

Our analysis employs the Kullback--Leibler divergence because Girsanov's theorem naturally provides an explicit expression for the Radon--Nikodym derivative between the path measures. It would be interesting to investigate whether similar approximation guarantees can be established under other probability metrics, including

\[
W_p(\mu,\mu_\theta), \qquad
\mathrm{TV}(\mu,\mu_\theta), \qquad
D_{\mathrm{JS}}(\mu,\mu_\theta),
\]

where $W_p$ denotes the Wasserstein distance, $\mathrm{TV}$ denotes the total variation distance, and $D_{\mathrm{JS}}$ denotes the Jensen--Shannon divergence. Such results may provide stronger geometric guarantees for generative modeling.

\subsection{Approximation Rates}

Hornik's theorem guarantees the existence of neural network approximations but does not quantify the approximation rate. Combining our framework with quantitative approximation theory for neural networks may lead to explicit convergence rates of the form

\[
\KL(\mu\|\mu_\theta)
=
O(N^{-\alpha}),
\]

where $N$ denotes the network size. Such results would provide a theoretical characterization of the expressive power of neural score models.

\subsection{General Diffusion Models}

The present analysis considers continuous-time score-based diffusion models represented by stochastic differential equations. Since denoising diffusion probabilistic models (DDPMs) arise as discrete approximations of these SDEs, it would be interesting to derive analogous approximation theorems directly for discrete-time diffusion processes.

Another promising direction is to extend the analysis to probability flow ordinary differential equations, where the reverse diffusion is replaced by a deterministic transport equation.

\subsection{Beyond Euclidean Data}

The current framework assumes that the data lie in the Euclidean space $\mathbb{R}^d$. Extending the analysis to diffusion models defined on manifolds, graphs, or other structured state spaces remains an important direction for future research. Such extensions may provide theoretical foundations for geometric diffusion models and scientific machine learning applications.

\subsection{Learning Theory for Diffusion Models}

The present paper establishes an approximation theorem for score-based diffusion models. A natural next step is to study statistical learning guarantees, including sample complexity, estimation error, and generalization error of neural score estimators. Developing a complete statistical learning theory for diffusion models remains an important open problem.
\bibliographystyle{plain}
\bibliography{isitbib}
\end{document}